\documentclass{article}

\usepackage{fullpage}
\usepackage{amsmath,amssymb,amsfonts}
\usepackage{graphicx}

\def\R{{\mathbb{R}}}
\def\pr{{\rm Pr}}
\def\E{{\mathbb E}}
\def\X{{\mathcal X}}

\def\B{{\mathcal B}}
\def\yh{{\widehat{y}}}

\def\dist{{\rm dist}}
\def\sign{{\rm sign}}
\def\vol{{\rm vol}}
\def\PL{{\mbox{\rm PL}}}

\newtheorem{thm}{Theorem}
\newtheorem{lemma}[thm]{Lemma}
\newtheorem{cor}[thm]{Corollary}

\newtheorem{defn}[thm]{Definition}

\newenvironment{proof}{\noindent {\sc Proof:}}{$\Box$ \medskip}

\title{Active learning using region-based sampling}

\author{Sanjoy Dasgupta \and Yoav Freund}

\begin{document}

\maketitle

\begin{abstract}%
We present a general-purpose active learning scheme for data in metric spaces. The algorithm maintains a collection of neighborhoods of different sizes and uses label queries to identify those that have a strong bias towards one particular label; when two such neighborhoods intersect and have different labels, the region of overlap is treated as a ``known unknown'' and is a target of future active queries. We give label complexity bounds for this method that do not rely on assumptions about the data and we instantiate them in several cases of interest.
\end{abstract}

\maketitle

\section{Introduction}

In \emph{active learning}, the starting point is a data set whose labels are hidden but can be obtained individually for a price. The goal is to label the data set, or to find a good classifier, at low cost.

We consider a formulation in which we have a collection of $n$ points $X = \{x_1, \ldots, x_n\}$ that lie in a metric space $(\X, d)$. We can request the label of any of these points $x$, in which case we get a value $y \in \{-1, +1\}$ with conditional distribution
$$ \eta(x) \ = \ \E[y | x] \ = \ \pr(y=1|x) - \pr(y=-1|x) $$
for some unknown function $\eta: \X \to [0,1]$. The \emph{Bayes-optimal} label for $x$, which we shall denote $g^*(x)$, is $-1$ if $\eta(x) < 0$ and $+1$ if $\eta(x) > 0$; either label can be used if $\eta(x) = 0$.
%$$ g^*(x) = \left\{
%\begin{array}{cl}
%-1 & \mbox{if $\eta(x) < 0$} \\
%+1 & \mbox{if $\eta(x) > 0$} \\
%\mbox{either} & \mbox{if $\eta(x)= 0$}
%\end{array} \right.
%$$
We wish to find the Bayes-optimal labels for the given data $X$.

More precisely, at the outset we have: the data $X$; two parameters $0 < \gamma, \delta < 1$; and a \emph{query budget}, the number of label queries we can make. We want a procedure that chooses the next point, or batch of points, to query. This process will be applied repeatedly until the query budget is exhausted, whereupon labels $\yh(x)$ must be provided for all $x \in X$, including those that were queried. Ideally, these will equal the Bayes-optimal labels $g^*(x)$, but we will only be judged on points $x \in X$ with $|\eta(x)| \geq \gamma$. That is, the number of mistakes is taken to be:
$$ \sum_{x \in X} {\bf 1}(\yh(x) \neq g^*(x),\ |\eta(x)| \geq \gamma) .$$
The overall procedure is allowed to fail with probability $\delta$, to account for sampling error.

\subsection{Nonparametric active learning}

In our setting there is no underlying assumption about $g^*$, for instance that it follows a linear model. Thus we adopt a nonparametric approach.

Existing proposals for nonparametric active learning can mostly be grouped according to the overall principle they follow: either (1) they seek to obtain the Bayes-optimal labels of a few well-positioned points, and then propagate these to the rest of the space \cite{DNZ15,H17,ASU18} or (2) they estimate the biases (positive or negative) of entire regions at a time \cite{DH08,M12}. In this paper, we follow the second strategy because the only reliable and general-purpose way to assess the sign of an individual point---that is, $\mbox{sign}(\eta(x))$---is to query that point repeatedly; absent smoothness assumptions, the sign can change abruptly in an arbitrarily small neighborhood around $x$. On the other hand, the sign of a region $B$---that is, the sign of the average $\eta$ value in $B$---is easy to determine, as long as $X$ contains a reasonable number of points from that region.

\subsection{Three key challenges}

The overall active learning strategy is to define a large collection of regions, or \emph{neighborhoods}, $\B$, of varying sizes, and to use label queries to estimate the signs associated with these regions. Ultimately we hope to cover $\X$ by a patchwork of neighborhoods with $\eta(x)$ values of uniform sign. Without strong smoothness conditions, these neighborhoods could be of very different sizes. The most beneficial for active learning are neighborhoods that contain a lot of $X$, but to get good coverage we might have to include smaller neighborhoods. To address this heterogeneity, we group the regions $\B$ by size. We begin by estimating the signs of the largest of them, and then move on to progressively smaller neighborhoods as the need arises. When the label budget runs out, the process is stopped, and provisional labels are assigned to the individual points in $X$.

In giving shape to this scheme, there are three key challenges to be addressed.

The first challenge is deciding \emph{where to query}. Suppose that a particular region $B \subset \X$ has an average $\eta$-value close to zero. Earlier work has typically taken this as a sign that $B$ is part of the ``uncertainty'' region and should be queried further. But it is important to distinguish two cases: (1) the $\eta$ values are close to zero throughout $B$, and (2) $B$ consists of two sub-regions, one of which is strongly positive while the other is strongly negative. In the first case, there is little merit in querying further. But in the second case, there is a lot to be gained.

To distinguish these two cases, we use a collection of neighborhoods that are \emph{overlapping}. For instance, we might take $\B$ to be \emph{all} balls in $\X$, which is effectively a finite collection once the given data points $X$ are taken into account. Case (2) can then be detected: we think of a point $x$ as being in the uncertainty region if it is contained in a neighborhood $B$ that is strongly positive as well as being in a neighborhood $B'$ that is strongly negative. Such points are ``known unknowns'', and these are the targets of our active querying.

\begin{figure}
\begin{center}
\includegraphics[width=3.5in]{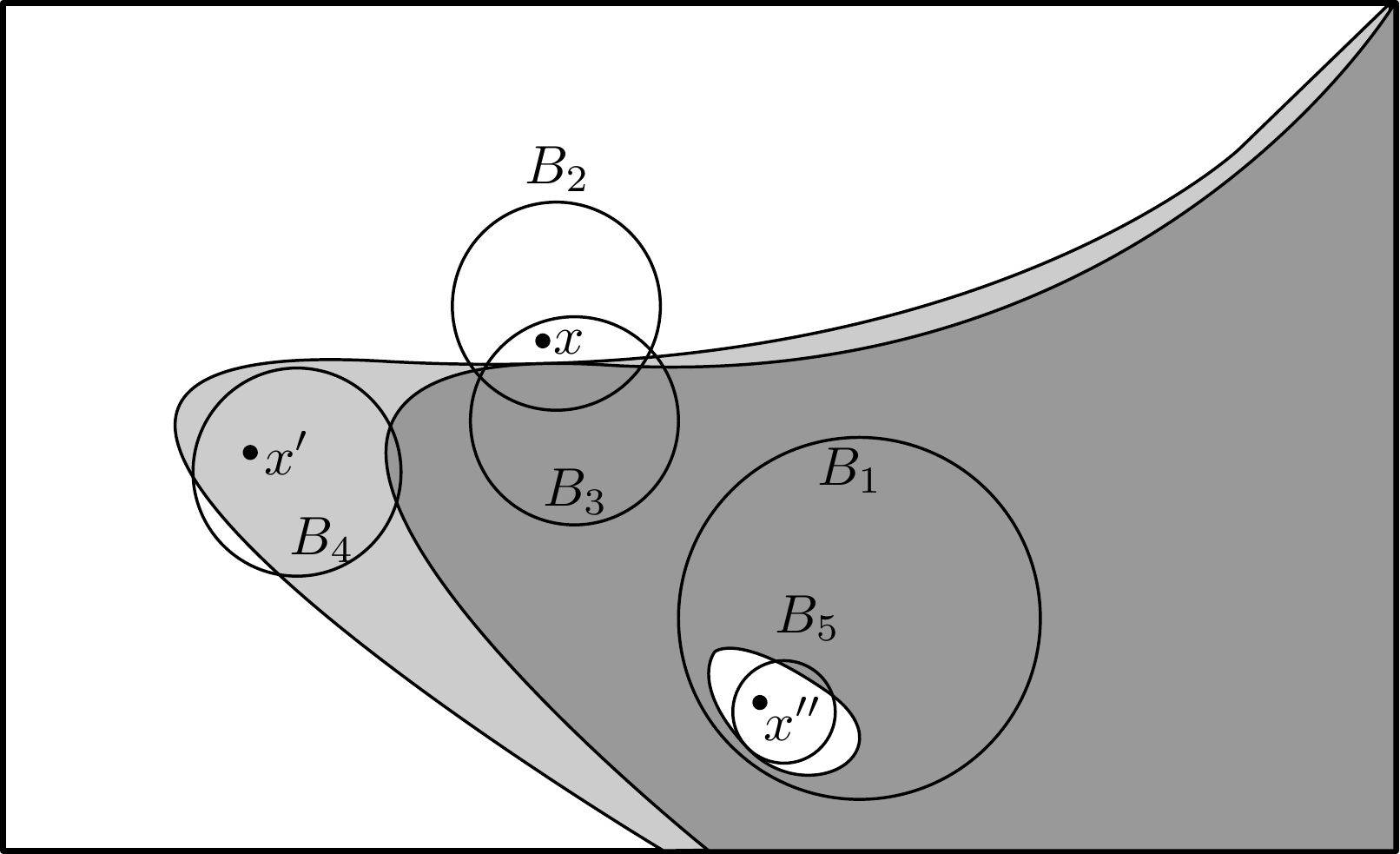}
\end{center}
\caption{The white region is strongly $-$, the dark region is strongly $+$, and grey is in between. Various neighborhoods $B_i \in \B$ are shown. Point $x$ lies in both a strongly positive region $B_3$ and a strongly negative region $B_2$ and is thus targeted for active querying. Point $x'$, which lies in a grey zone, is not targeted. Point $x''$ initially seems to have label $+$, when $B_1$ is sampled, but later has a mind-change to $-$, when $B_5$ is sampled.}
\label{fig:challenges}
\end{figure}

A second challenge is that the sign of $\eta(\cdot)$ can change in an arbitrarily small neighborhood of any given point $x$. Thus, if we look at successively smaller neighborhoods containing $x$, say $B_1 \supset B_2 \supset B_3 \supset \cdots$, the labels of these sets may keep changing. At first, when we are querying neighborhoods of the size of $B_1$, we might think $x$ has label $+1$. When we move to neighborhoods of the size of $B_2$, this could change to $-1$. And then $+1$, and so on. We call these \emph{mind changes}.

What this means is that we can \emph{never} be sure of having correctly determined the label of $x$. Thus in addition to focused (active) querying, we also do background sampling of the entire space to pick up on possible mind changes. For simplicity, we make one background query per focused query. 

The third challenge is \emph{managing sampling} of overlapping neighborhoods. Recall that we are interested in detecting points, and thus neighborhoods, whose $\eta$-values are either $> \gamma$ or $<-\gamma$. This suggests querying $k \approx 1/\gamma^2$ points at random from each region. Now, suppose we have queried this many points from neighborhood $B$ and then later want to query points from a different neighborhood $B'$ that overlaps $B$. How can we reuse the queries we have already made in $B \cap B'$?

\emph{Poisson sampling} provides a clean solution. Rather than choosing $k$ points at random from $B$, we pick each point in $B$ with probability (roughly) $k/|B \cap X|$, independently. The specific way we implement this is to assign each $x \in X$ a uniform-random value $T_x \in [0,1]$, at the outset. When sampling $B$, we choose to query $x$ if $T_x \leq k/|B \cap X|$. And when it comes time to sample a different $B'$ that also contains $x$, we choose $x$ if $T_x \leq k/|B' \cap X|$; if its label has already been obtained, we are able to reuse it. In this way, the random querying of overlapping regions is seamlessly managed.

These three challenges go beyond earlier work in active learning, which was able to avoid problems like mind-changes by making smoothness assumptions on $\eta$. By tackling all three of them, we are able to give a general-purpose active learning scheme. 

\subsection{Results}

%Our algorithm does not require any assumption on the input distribution. In the worst case the number of queries is twice the number of queries of a passive algorithm. We show that under several natural conditions the query complexity of our algorithm is logarithmic in the error.

In Section~\ref{sec:algorithm}, we present our active learning algorithm, and in Section~\ref{sec:discrete}, we analyze it in generality, taking $X$ to be an arbitrary point-set in the metric space and allowing any $\eta$ function. We identify two \emph{critical levels} for any $x \in X$: two scales (of neighborhood sizes) that control how many queries are sufficient for $x$ to be correctly labeled (Theorem~\ref{thm:label-complexity}). We instantiate these bounds in canonical one-dimensional settings (Theorems~\ref{thm:oned-massart} and \ref{thm:oned-monotonic}) to get label complexities logarithmic in $|X|$.

In Section~\ref{sec:continuous}, we consider the statistical setting where $X$ is drawn from an underlying distribution $\mu$ on $\X$. We give rates of convergence under distributional conditions (Theorem~\ref{thm:label-complexity-specific}), and instantiate these under assumptions that are common in learning theory and computational geometry.

\subsection{Related work}

There is a small body of work on the theory of nonparametric active learning. The early results of \cite{CN08} established upper and lower bounds on label complexity in situations where the Bayes-optimal boundary is of a simple form: a single threshold for one-dimensional data or a ``smooth boundary fragment'' in higher dimension.

An algorithm for active learning based on hierarchical sampling was given by \cite{DH08} and was analyzed under smoothness conditions by \cite{KUB15}. In these works, the idea is to begin with a hierarchical clustering of $X$, and to then use queries to discover a pruning of this tree whose leaf-clusters are almost-pure in their labels. The method is not well-suited to situations with significant noise levels. Another approach using dyadic partitions was given by \cite{M12} and analyzed under commonly-used smoothness, margin, and density assumptions---namely, that $\eta$ is Holder-smooth, the fraction of points with $|\eta(x)| \leq t$ is some polynomial in $t$, and the marginal density is close to uniform---along with an additional smoothness requirement on $\eta$.

A different strategy using nearest neighbors was explored by \cite{ASU18}. Their idea was to choose an appropriate scale $s$, find an $s$-cover of $X$, estimate the Bayes-optimal label for each point in this cover by querying its neighbors, and then use these cleanly-labeled points for 1-nearest neighbor classification. A somewhat more general approach was given by \cite{H17} and studied under the usual smoothness, margin, and density conditions, with resulting rates of convergence comparable to those found by \cite{M12}.

Finally, \cite{DNZ15} suggested a graph-based method for active learning based on adaptively looking for the cut in the graph corresponding to the correct decision boundary. Their assumptions are based on properties of this cut and are not easily comparable with earlier work.

\section{The active learning algorithm}
\label{sec:algorithm}

Before giving a high-level overview of our algorithm, we recall some basic notation. We have a collection of points $X$ lying in a metric space $(\X, d)$. The label of any $x \in X$ can be requested and the value returned will be either $-1$ or $+1$, according to the conditional probability function $\eta(x) = \E(y|x)$. We wish to assign Bayes-optimal labels to all points in $X$ with $|\eta(x)| \geq \gamma$.

\subsection{A collection of sampling regions}

In our active learning algorithm, sampling is organized around a collection $\B$ of subsets of $\X$. These are the atomic sets on which we assess label bias and we call them ``balls'' or ``neighborhoods''. There are no formal requirements on $\B$---for instance, the so-called balls can be of arbitrary shape---but the intention is that each $x \in X$ is contained in balls of multiple sizes, including one that is so small as to exclude the rest of $X$.

For any ball $B \in \B$, let $X_B = X \cap B$ be a shorthand for the data points that lie in it.
%\footnote{Since $X_B$ is used to assess the bias of $B$, we need to exclude points that cannot be considered random samples from $B$. For instance, consider these two alternatives: (1) $\B$ consists of a pre-defined set of balls. (2) $\B$ consists of balls defined by pairs of points in $X$, with each $x,x' \in X$ yielding the ball $B(x,\|x-x'\|)$. In the first case, all points $X \cap B$ are random draws from $B$. In the second case, we need to exclude the two points $x,x'$ that define the ball.} 
We group balls into levels by the number of points they contain. We put $B$ at {\it level} $\ell \geq 0$ if
\begin{equation}
\frac{n}{2^{\ell + 1}} \leq |X_B| < \frac{n}{2^\ell} .
\label{eq:sampling-level}
\end{equation}
Let $\B_\ell$ consist of all balls in $\B$ that are at level $\ell$. Thus $\B_0, \B_1, \ldots$ is a partition of $\B$, with $\B_0$ consisting of highly-populated balls and subsequent $\B_1, \B_2, \ldots$ consisting of successively smaller balls. We will use $\B_{\geq \ell}$ to denote all balls at levels $\ell$ or greater, and likewise $\B_{> \ell}$, $\B_{\leq \ell}$, and so on. 

Balls in lower levels contain more points, and thus their biases (average $\eta$ values) are easier to estimate. Our sampling algorithm makes its way from easier to harder levels.

For any $x \in \X$, let $\B(x) \subset \B$ denote the collection of balls that contain $x$ and can thus be used in determining $x$'s label. We again partition these balls by sampling-level, so that $\B_\ell(x) = \B(x) \cap \B_\ell$. 
%We will also refer to $\B_\ell(x)$ as the {\em level $\ell$ neighborhood of $x$}.

\subsection{Estimating bias}

We use label-queries to estimate the biases (average $\eta$ values) of balls $B \in \B$. These are in turn used to estimate the labels of individual points.

The bias of a ball $B \in \B$ is defined as
$$ \eta_X(B) = \mbox{average}\{\eta(x): x \in X_B \} .$$
Rather than working with a numerical estimate, we assign each ball a \emph{qualitative} bias estimate,
$$ \yh(B) = 
\left\{
\begin{array}{cl}
+1 & \mbox{significant $+$ bias} \\
-1 & \mbox{significant $-$ bias} \\
0 & \mbox{no significant bias} \\
\bot & \mbox{not yet available}
\end{array}
\right.
$$

The option $\bot$ is used until sufficiently many points in $X_B$ have been queried: the required number is $k = O(1/\gamma^2)$, recalling that $\gamma$ is the smallest bias that needs to be detected. Once this many labels are available, $\yh(B)$ is set to a value in $\{-1,0,+1\}$ and remains fixed thereafter. These bias estimates will with high probability be seen to satisfy the following guarantee.
\begin{defn}
For any $B \in \B$, bias estimate $\yh(B) \in \{+1,-1,0\}$ is \emph{$\gamma$-accurate} if:
\begin{itemize}
\item $\yh(B) = +1 \implies \eta_X(B) > 0$
\item $\yh(B) = -1 \implies \eta_X(B) < 0$
\item $\yh(B) = 0 \implies |\eta_X(B)| < \gamma$
\end{itemize}
\label{def:accurate-bias-estimate}
\end{defn}

Pick any point $x$ and any level $\ell$. Once qualitative bias estimates are available for all balls $B \in \B_\ell(x)$, the set of possible labels for $x$ at level $\ell$, denoted $\PL_\ell(x) \subset \{-1,+1\}$, is defined thus:
\begin{itemize}
\item $\PL_\ell(x)$ contains $+1$ if there exists a \emph{minimal} ball $B \in \B_{\leq \ell}(x)$ (that is, which has no other $B' \in \B_{\leq \ell}(x)$ strictly contained with it) with $\yh(B) = +1$.
\item $\PL_\ell(x)$ contains $-1$ under a symmetrical condition.
\end{itemize}
This is spelled out precisely in Equation~(\ref{eq:PL}) in  in Figure~\ref{alg:main}. The label-estimate for $x$ at level $\ell$, denoted $\yh_\ell(x)$, is $+1$ if $\PL_\ell(x)=\{+1\}$, $-1$ if $\PL_\ell(x) = \{-1\}$, $0$ if $\PL_\ell(x) = \{\}$, and $!$ if $\PL_\ell(x) = \{-1,+1\}$ (see Equation~(\ref{eq:provisional-label})). 
%The interpretation of $\yh_\ell(x) \in \{+1,-1,0\}$ is given in Definition~\ref{def:accurate-bias-estimate}. 
The label ``!'' can be interpreted as ``known unknown''~\cite{R11} or as ``conflicting evidence''. Our active learning algorithm makes all of its \emph{focused queries} in balls that contain known unknowns.

% to be $-1$ if $\PL_\ell(x) = %\{-1\}$, or $+1$ if $\PL_\ell(x) = \{+1\}$, or $0$ if $\PL_\ell(x) = \{\}$, or $!$ %if $\PL_\ell(x) = \{-1,+1\}$, %and remains fixed thereafter. Points labeled $!$ are ``known unknowns'': they are %near the decision boundary and %need further investigation. Points labeled $0$ show no significant bias.

\subsection{A neighborhood-based active learning algorithm}

The active learning algorithm is described in Figures~\ref{alg:main} and \ref{alg:sampling}. There are two types of queries: \emph{focused queries} and \emph{background queries}. Background queries are random draws from $X$ and correspond to passive learning. Focused queries are made in the vicinity of ``uncertain'' points and correspond to active learning. The \emph{uncertainty region} $U_\ell$ at level $\ell \geq 1$ consists of points $x \in X$ on which the prediction from the previous level is ``known unknown'', $\yh_{\ell-1}(x) = \,!$. Focused queries are drawn at random from balls in $\B_\ell$ that contain such points.

On each iteration of the main loop, (at most) one focused query is made as well as a background query. The focused query comes from the lowest-numbered uncertainty region $U_\ell$ that is nonempty. 

Once these queries are made, all bias estimates $\yh(B)$ are updated, along with label-sets $\PL_\ell(x)$, label-estimates $\yh_\ell(x)$, and finally the uncertainty regions $U_\ell$. Then the next iteration begins.

\begin{figure}[h!]
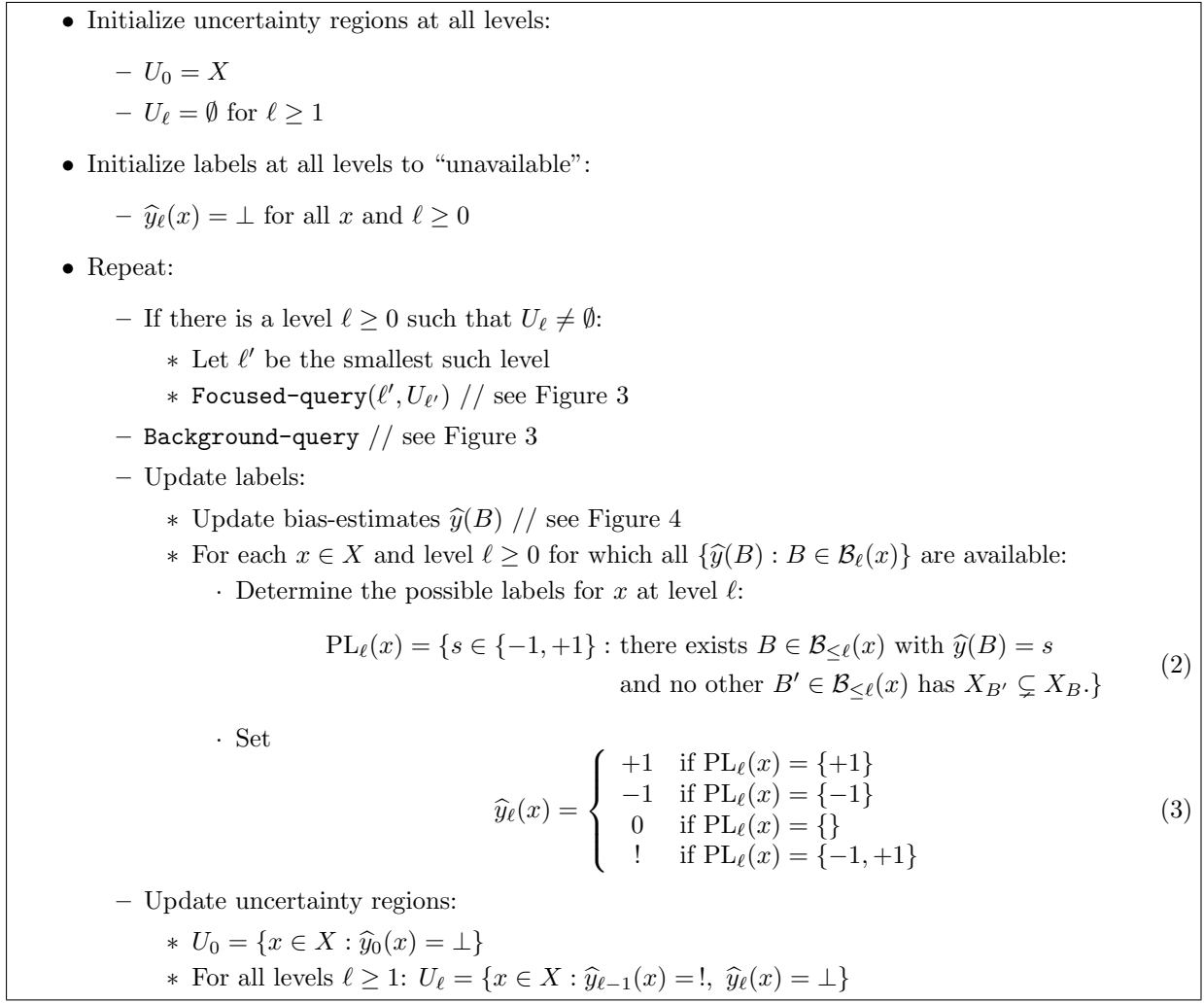

\framebox{
\begin{minipage}[t]{6.3in}
\begin{itemize}
\item Initialize uncertainty regions at all levels:
\begin{itemize}
\item $U_0 = X$
\item $U_\ell = \emptyset$ for $\ell \geq 1$
\end{itemize}
\item Initialize labels at all levels to ``unavailable'':
\begin{itemize}
\item $\yh_\ell(x) = \bot$ for all $x$ and $\ell \geq 0$
\end{itemize}
\item Repeat:
\begin{itemize}
\item If there is a level $\ell \geq 0$ such that $U_\ell \neq \emptyset$:
\begin{itemize}
\item Let $\ell'$ be the smallest such level
\item {\tt Focused-query}($\ell', U_{\ell'}$) // see Figure~\ref{alg:sampling}
\end{itemize}
\item {\tt Background-query} // see Figure~\ref{alg:sampling}
\item Update labels:
\begin{itemize}
\item Update bias-estimates $\yh(B)$ // see Figure~\ref{alg:bias-estimate}
\item For each $x \in X$ and level $\ell \geq 0$ for which all $\{\yh(B): B \in \B_\ell(x)\}$ are available:
\begin{itemize}
\item Determine the possible labels for $x$ at level $\ell$:
\begin{equation} 
\begin{split}
\PL_\ell(x) = \{s \in \{-1, +1\}:\  & \mbox{there exists $B \in \B_{\leq \ell}(x)$ with $\yh(B) = s$} \\ 
& \mbox{and no other $B' \in \B_{\leq \ell}(x)$ has $X_{B'} \subsetneq X_B$.} \} 
\end{split}
\label{eq:PL}
\end{equation}
\item Set
\begin{equation} 
\yh_\ell(x) = 
\left\{
\begin{array}{cl}
+1 & \mbox{if $\PL_\ell(x) = \{+1\}$} \\
-1 &  \mbox{if $\PL_\ell(x) = \{-1\}$} \\
0 & \mbox{if $\PL_\ell(x) = \{\}$} \\
! & \mbox{if $\PL_\ell(x) = \{-1,+1\}$}
\end{array}
\right.
\label{eq:provisional-label}
\end{equation}
\end{itemize}
\end{itemize}
\item Update uncertainty regions:
\begin{itemize}
\item $U_0 = \{x \in X: \yh_0(x) = \bot\}$
\item For all levels $\ell \geq 1$: $U_\ell = \{x \in X: \yh_{\ell-1}(x) = \,!, \ \yh_\ell(x) = \bot\}$
%\item YF: $U_\ell = \bigcup \{B, B \in \B_\ell, \exists x \in B: \yh_{\ell-1}(x) = \,!, \ \yh_\ell(x) = \bot\}$
\end{itemize}
\end{itemize}
\end{itemize}

\end{minipage}}
\caption{The active learning algorithm. Each iteration of the main loop makes (at most) one focused query and one background query.}
\label{alg:main}
\end{figure}

The querying process can be stopped at any time, whereupon labels are assigned as follows:
\begin{equation}
\yh(x) = 
\left\{
\begin{array}{ll}
\yh_\ell(x) & \mbox{for the largest $\ell$ with $\yh_\ell(x) \in \{-1,+1\}$, if such an $\ell$ exists} \\
0 & \mbox{(meaning ``don't know''), otherwise}
\end{array}
\right.
\label{eq:final-label}
\end{equation}

\begin{figure}[h!]
\framebox{
\begin{minipage}[t]{6.3in}

\vspace{.05in}
\emph{Initialization:}
\begin{itemize}
\item Set $Q = \emptyset$ (points queried so far)
\item For each $x \in X$: choose $T_x \sim \mbox{uniform}([0,1])$
\end{itemize}

\vspace{.1in}
{\bf Focused-query}($\ell, U$)

%\vspace{.1in}
\begin{itemize}
\item Define querying region:
$$ S =  \bigcup_{x \in U} \bigcup_{B \in \B_\ell(x)} \{z \in X_B: T_z \leq \tau_\ell\} $$
\item Query the next unlabeled point in $S \setminus Q$, ordered by $T_z$ values, and add to $Q$
\end{itemize}

\vspace{.1in}
{\bf Background-query}

%\vspace{.1in}
\begin{itemize}
\item Query the next unlabeled point in $X \setminus Q$, ordered by $T_z$ values, and add to $Q$
\end{itemize}
\end{minipage}}
\caption{The two sampling procedures. Each $x \in X$ has an associated value $T_x$ chosen uniformly from $[0,1]$. This smaller this value, the earlier $x$ is likely to be queried. Focused querying uses level-based thresholds $\tau_\ell = \min(2^{\ell+2}k/n, 1)$, where $k$ is a global parameter.} 
\label{alg:sampling}
\end{figure}

\begin{figure}[h!]
\framebox{
\begin{minipage}[t]{6.3in}
\vspace{.05in}
\begin{itemize}
\item Initially $\yh(B) = \bot$
\item When all of $\{z \in X_B: T_z \leq \tau_\ell\}$ is queried, let $\widehat{\eta}(B)$ be the mean of these labels and set
$$ \yh(B)
= 
\left\{
\begin{array}{ll}
\sign(\widehat{\eta}(B)) & \mbox{if $|\widehat{\eta}(B)| \geq \gamma/2$} \\
0 & \mbox{otherwise}
\end{array}
\right.
$$
\end{itemize}
\end{minipage}}
\caption{The qualitative bias $\yh(B)$ of a ball $B \in \B_\ell$.}
\label{alg:bias-estimate}
\end{figure}

\section{Analysis of algorithm: finite population setting}
\label{sec:discrete}

We now analyze the active learning procedure in a setting where $X \subset \X$ is an arbitrary set of $n$ points; that is, we make no distributional assumption on the manner in which $X$ is generated.

\subsection{Accuracy of bias estimates}

Fix the set of balls $\B$ and let $0 < \delta < 1$ be a
predefined confidence parameter. We start with a uniform guarantee on the
bias estimates for all balls $B \in \B$.

In Figure~\ref{alg:sampling}, we see that the query region for any ball $B \in \B_\ell$ is $\{z \in X_B: T_z \leq \tau_\ell\}$, which from the definition of $\tau_\ell$ has size $O(k)$. This is the number of queries we make to $B$ before estimating its qualitative bias. Since we need to detect biases of magnitude greater than $\gamma$, we would expect $k$ to be proportional to $1/\gamma^2$. This intuition is borne out by the following result, proved in the appendix. 
\begin{lemma}
Suppose that $k \geq (192/\gamma^2) \ln (4 |\B|/\delta)$. 
Let $\yh(B)$ be defined as in Figure~\ref{alg:bias-estimate}. Then with probability $\geq 1-\delta$, all the $\yh(B)$, for $B \in \B$, are $\gamma$-accurate in the sense of Definition~\ref{def:accurate-bias-estimate}.
\label{lemma:accurate-bias-estimates}
\end{lemma}
In what follows, we will assume that $\yh(B)$ is $\gamma$-accurate for all  $B \in \B$.

\subsection{Critical levels}

The label assigned to a data point $x$ at level $\ell$, denoted $\yh_\ell(x)$, can change as $\ell$ increases; it may flip between $+1$ and $-1$, with stretches of $0$ or $!$ in between. This $\yh_\ell(x)$ is governed by $\PL_\ell(x) \subset \{-1,+1\}$, the ``possible labels'' for $x$ given the information from balls at levels $0$ through $\ell$. The value of $\PL_\ell(x)$ depends upon the random choices of the querying algorithm, but it is nonetheless possible to define two critical levels for each $x$: a level $L_1(x)$ by which $\PL_\ell(x)$ will reliably contain the correct label of $x$, and a level $L_2(x)$ by which $\PL_\ell(x)$ will reliably omit the wrong label.

\begin{defn}[Critical levels $L_1,L_2$]
Pick any $x \in X$ with $\eta(x) \neq 0$ and let $s(x) = \sign(\eta(x))$ be its Bayes-optimal label. We define $L_1(x)$ to be the smallest level $\ell$ such that:
\begin{itemize}
\item There exists some $B_o \in \B_\ell(x)$ with $s(x) \cdot \eta_X(B_o) \geq \gamma$.
\item Any $B \in \B(x)$ with $X_{B} \subset X_{B_o}$ also has $s(x) \cdot \eta_X(B) \geq \gamma$.
\end{itemize}
We define $L_2(x)$ to be the smallest level $\ell$ such that:
\begin{itemize}
\item For all $B \in \B_{\geq \ell}(x)$, we have $s(x) \cdot \eta_X(B) \geq 0$.
\item For any $B \in \B_{\leq \ell}(x)$ with $s(x) \cdot \eta_X(B) < 0$, there exists $B' \in \B_{\leq \ell}(x)$ with $X_{B'} \subset X_B$ and $s(x) \cdot \eta_X(B') \geq 0$.
\end{itemize}
Take $L_1(x)$ or $L_2(x)$ to be $\infty$ if no level meets the requirements. 
\label{defn:L12}
\end{defn}

The significance of these definitions is made clear by the following lemma.
\begin{lemma}
Pick any $x \in X$ with $\eta(x) \neq 0$ and let $s(x) \in \{+1,-1\}$ denote its Bayes-optimal label. Then for any level $\ell$ and any time at which $\yh_\ell(x) \neq \bot$:
\begin{enumerate}
\item[(a)] If $\ell \geq L_1(x)$, then $s(x) \in \PL_\ell(x)$ and thus $\yh_\ell(x) \in \{s(x), !\}$. 
\item[(b)] If $\ell \geq L_2(x)$, then $-s(x) \not\in \PL_\ell(x)$ and thus $\yh_\ell(x) \in \{s(x), 0\}$.
\end{enumerate}
\label{lemma:boundary}
\end{lemma}

\subsection{Boundary sets and label complexity}

A common intuition about active learning is that successive queries gradually constrain the possible locations of the decision boundary. Let's consider the state of affairs when all balls at level $\leq \ell - 1$ have been sampled. The ``known unknowns'' at level $\ell$ are points $x$ with $\yh_{\ell-1}(x) = \, !$; by Lemma~\ref{lemma:boundary}(b), such points have $L_2(x) \geq \ell$. Focused sampling at level $\ell$ will be restricted to balls that contain these points. We can think of this region as the \emph{boundary set} at level $\ell$.
\begin{defn}[Boundary set $\bf \Delta_\ell$]
For any level $\ell$, define the \emph{boundary set} at level $\ell$ to be
\begin{equation}
\Delta_\ell 
= \bigcup_{x \in X: L_2(x) \geq \ell} \bigcup_{B \in \B_{\ell}(x)} X_B 
.
\label{eq:sampling-region}
\end{equation}
\end{defn}
%It can be shown that all focused samples at level $\ell$ must lie in this set.
\begin{lemma}
All focused samples at level $\ell$ lie in $\{z \in \Delta_\ell: T_z \leq \tau_\ell \}$.
\label{lemma:focused}
\end{lemma}

We can now give generic label complexity bounds in terms of $L_1$, $L_2$, and $\Delta_\ell$. The bounds come in two equivalent forms: a \emph{global} version that specifies what parts of $X$ are correctly labeled after $m$ queries and a \emph{local} version that specifies the number of queries after which a particular $x$ is correctly labeled. The global formulation (Theorem~\ref{thm:label-complexity-0}) is in the Appendix; here is the local version.

\begin{thm}
Suppose $k \geq (192/\gamma^2) \ln (4 |\B|/\delta)$. Then with probability at least $1-2\delta$, the following holds for all $x \in X$. Let $L_1(x)$ and $L_2(x)$ be the critical levels for $x$, as in Definition~\ref{defn:L12}. If $L_2(x) \leq \lg (n/2k)$, let
\[
m_o(x) = 32k \cdot \max\bigg( 2^{L_1(x)}, \ \frac{1}{n} \sum_{\ell=0}^{L_2(x)} |\Delta_\ell| \, 2^\ell\bigg) .
\]
After the active learning algorithm has made $m_o(x)$ queries, $\yh(x)$ will remain fixed at the Bayes-optimal label $g^*(x)$.
\label{thm:label-complexity}
\end{thm}

The argument for Theorem~\ref{thm:label-complexity} is roughly that during the first $m_o(x)/2$ queries, background sampling alone is enough to ensure that $\yh_\ell(x)$ is forever set to either the correct label or !. During the next $m_o(x)/2$ queries, focused sampling then correctly resolves the label. 

One important feature of our algorithm is that it queries any point at most once. However, in some applications, a point $x$ can be queried repeatedly, with each resulting label being an independent draw according to $\eta(x)$. If \emph{repeat queries} are permissible, $O(1/\gamma^2)$ copies should be made of each point in $X$ before the algorithm is applied; and in this case, points with $|\eta(x)| \geq \gamma$ will all be correctly labeled, eventually.

\subsection{Example: One-dimensional data with monotonic $\eta$}

In order to apply Theorem~\ref{thm:label-complexity}, we need upper bounds on the critical levels $L_1(x)$ and $L_2(x)$ for each point $x \in X$, and upper bounds on the size of the sampling region $\Delta_\ell$ at each level $\ell$. We now illustrate how this works out in a canonical setting.

Suppose $X$ is an arbitrary set of $n$ points in $\X = [0,1]$ and is labeled according to a conditional probability function $\eta: \X \to [-1,1]$ that is continuous and strictly increasing. Let $\lambda \in (0,1)$ be the point for which $\eta(\lambda) = 0$ and let $\lambda_L, \lambda_R$ be the points for which $\eta(\lambda_L) = -\gamma$ and $\eta(\lambda_R) = \gamma$. Thus we are not required to label points in the interval $(\lambda_L, \lambda_R)$. See Figure~\ref{fig:oned}(a) for an example.

\begin{figure}
\begin{center}
\includegraphics[width=2.5in]{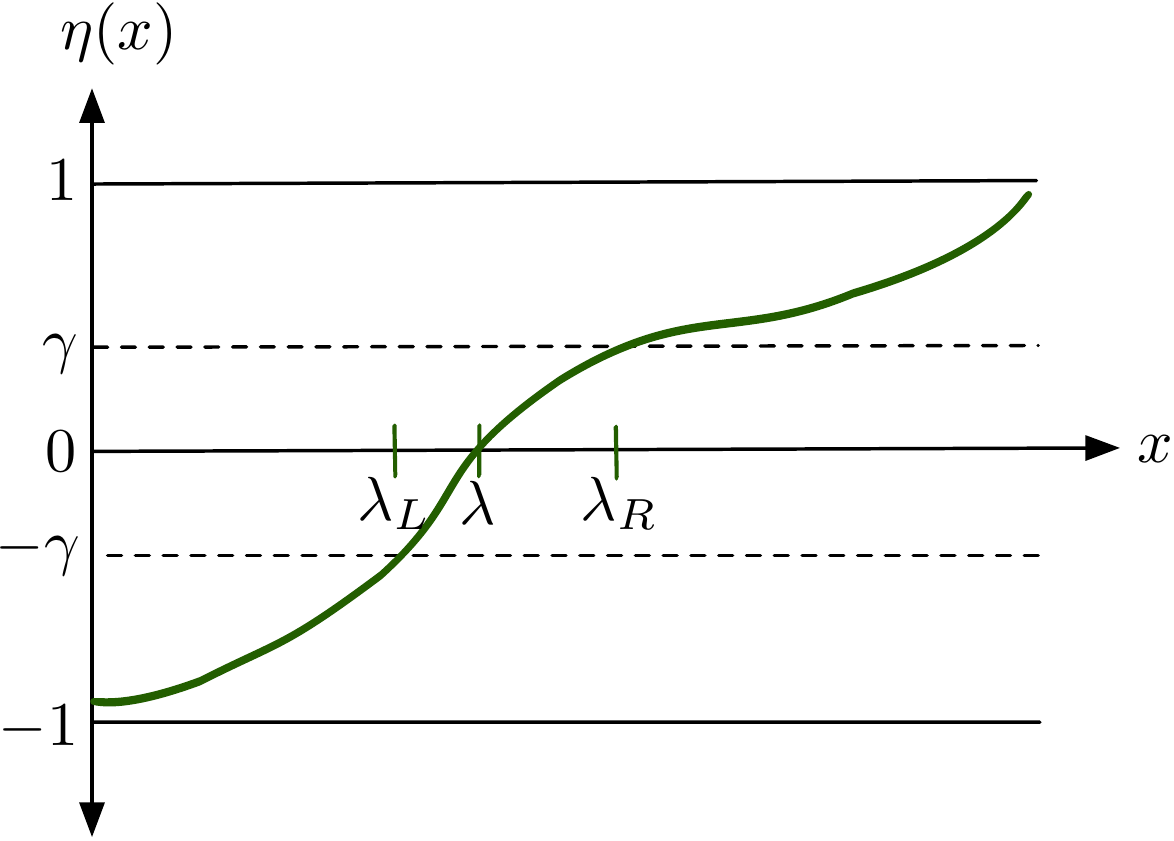}
\hspace{.5in}
\includegraphics[width=2.5in]{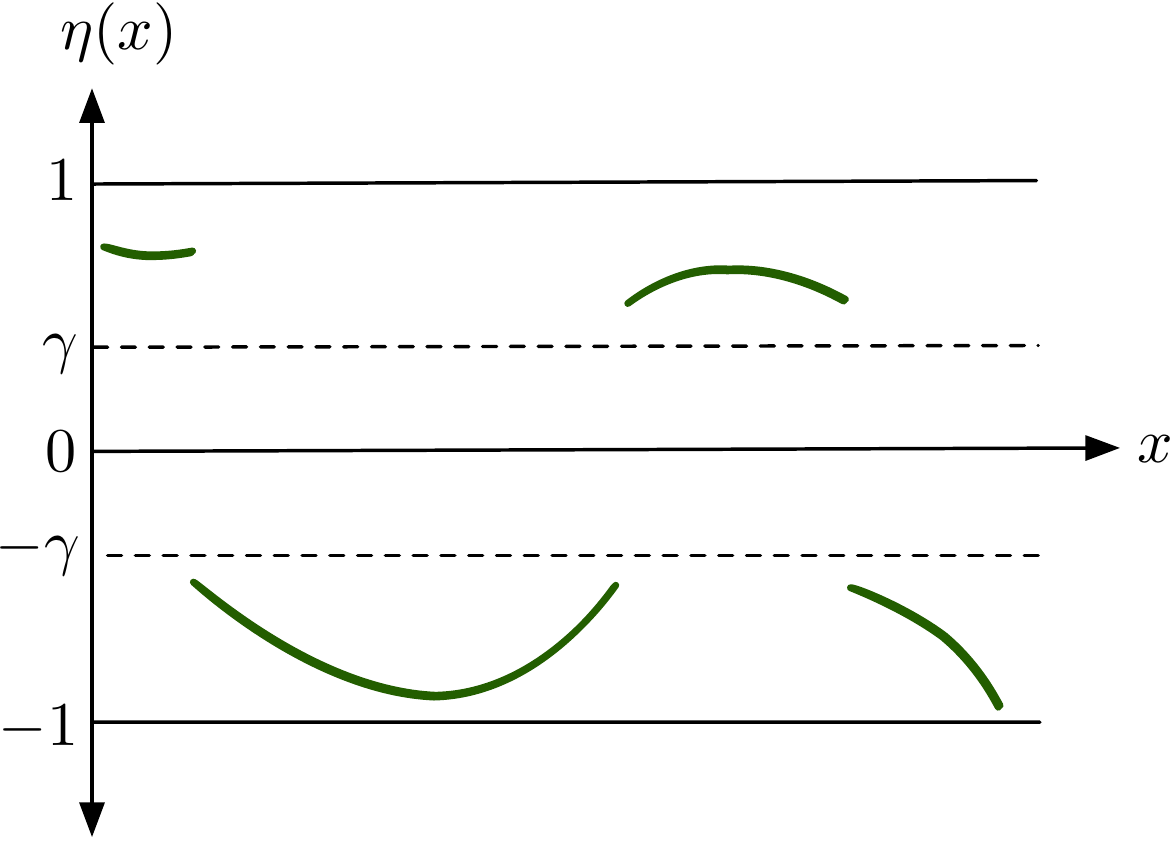}
\end{center}
\caption{(a) The conditional probability function $\eta$ is monotonically increasing; $\eta(x) = -\gamma, 0, \gamma$ at $x = \lambda_L, \lambda, \lambda_R$, respectively. (b) On each subinterval, $\eta$ is either $> \gamma$ or $<-\gamma$.}
\label{fig:oned}
\end{figure}

%\begin{figure}
%\begin{center}
%\includegraphics[width=3in]{oned-monotonic.pdf}
%\end{center}
%\caption{The conditional probability function $\eta$ is monotonically increasing. The values $x = %\lambda_L, \lambda, \lambda_R$ have $\eta(x) = -\gamma, 0, \gamma$ respectively.}
%\label{fig:oned-monotonic}
%\end{figure}

Take $\B$ to consist of all closed intervals, and $\B(x)$ to be intervals containing $x$. It can be shown that for $n^- = |[0,\lambda_L] \cap X|$ and $n^+ = |[\lambda_R,1] \cap X|$, 
\begin{itemize}
\item $L_1(x) \leq \lg (n/n^+)$ if $x \geq \lambda_R$ and $L_1(x) \leq \lg (n/n^-)$ if $x \leq \lambda_L$, and
\item $L_2(x) \leq \lg (n/r(x))$, where $r(x)$ is the number of data points between $x$ and $\lambda$, counting $x$ as well.
\end{itemize}
A simple counting argument then shows that the boundary region is exponentially shrinking, $|\Delta_\ell| \leq 4n/2^\ell$, giving the following label complexity bound.
\begin{thm}
Define $r^+ = \min \{r(x): x \in X \cap [\lambda_R,1]\}$ and $r^- = \min \{r(x): x \in X \cap [0,\lambda_L]\}$. Pick any $0 < \delta < 1$. Suppose we run the algorithm of Figure~\ref{alg:main} with $k = O((\log (n/\delta))/\gamma^2)$ and that $\min(r^+, r^-) \geq 2k$.  Take any 
$$ m \geq 64k \cdot \max \left( \frac{n}{\min(n^+,n^-)}, \ 2 \lg \frac{\min(n^+,n^-)}{\min(r^+,r^-)} \right) .$$
With probability at least $1-\delta$, after making $m$ queries, the algorithm will correctly label all points $x \in X$ with $|\eta(x)| \geq \gamma$.
\label{thm:oned-monotonic}
\end{thm}
For $n^+, n^-, r^+, r^- = \Theta(n)$, this yields a label complexity of $O(\log n)$.

\subsection{Example: One-dimensional data with Massart noise}

In the Appendix (Section~\ref{sec:oned-massart}), we develop another one-dimensional example, in which $\X$ consists of several pieces, each of which either has $\eta > \gamma$ or $\eta < -\gamma$; see Figure~\ref{fig:oned}(b) for an example. Once again, a logarithmic label complexity is obtained.

\section{Analysis of algorithm: distributional setting}
\label{sec:continuous}

We now turn to a setting where the points $X$ are sampled from a distribution $\mu$ on $\X \subset \R^d$. The algorithm is unchanged, but we seek to bound label complexity in terms of properties of $\mu$ and $\eta$.

For simplicity, we focus on the case where $\mu$ is absolutely continuous with respect to the Lebesgue measure on $\R^d$ and thus admits a density. We take $\B$ to consist of all open balls $B(x,r) = \{z: \|z-x\| < r\}$ centered in $\X$, with $\B(x)$ being balls that contain $x$.

\subsection{A generic bound based on a probabilistic notion of distance}

Our analysis rests upon a notion of distance based on the distribution $\mu$. We take the distance from a point $x$ to a set $S$ to be the probability mass of the smallest ball that contains $x$ and touches $S$.
\begin{defn}
For any $x \in \X$ and $S \subset \X$, define
$\dist(x,S) = \inf \{\mu(B): B \in \B(x), B \cap S \neq \emptyset\} .$
\label{defn:prob-dist}
\end{defn}
We will see that $L_2(x)$ can be bounded in terms of such distances. In what follows, for $s \in \{-1,0,+1\}$, take $\X^s$ to denote $\{x \in \X: \sign(\eta(x)) = s\}$.
\begin{lemma}
For any $x \in \X$ with $\eta(x) \neq 0$, let $s(x) = \sign(\eta(x))$. Then
$$ L_2(x) \leq \bigg\lceil \lg \frac{1}{\dist(x, \X^{-s(x)})} \bigg\rceil + 1.$$
\label{lemma:L2-bound}
\end{lemma}

We will also use this probabilistic distance to define notions of boundary. We take the $p$-boundary to consist of points that are at distance $\leq p$ from points of the opposite label.
\begin{defn}
For any $p > 0$, define $\partial_p = \{x \in \X: \dist(x, \X^{-s(x)}) \leq p\}$.
\label{defn:boundary}
\end{defn}
Notice that this includes all $x$ with $\eta(x) = 0$.

The $(p,q)$-boundary consists of points at distance $\leq q$ from the $p$-boundary.
\begin{defn}
For any $p,q > 0$, define $\partial_{p,q} = \{x \in \X: \dist(x, \partial_p) \leq q\}$.
\label{defn:boundary2}
\end{defn}

We can bound $\Delta_\ell$, the query region at level $\ell$, in terms of this second-order boundary.
\begin{lemma}
For any $\ell \geq 0$, we have $\Delta_\ell \subset \partial_{4/2^\ell, 2/2^\ell}$.
\label{lemma:delta-continuous}
\end{lemma}

Theorem~\ref{thm:label-complexity} thus continues to hold, but with $n \cdot \mu(\partial_{4/2^\ell, 2/2^\ell})$ in place of $|\Delta_\ell|$; see Theorem~\ref{thm:label-complexity-dist} in the Appendix. 

\subsection{Bounds under three assumptions}

In order to get concrete label complexity bounds in the distributional setting, we need to be able to bound the levels $L_1(x)$ and $L_2(x)$, as well as the probability masses of boundary sets $\partial_{p,q}$. To do this, we introduce three assumptions.

 Let $\X_\gamma = \{x \in \X: |\eta(x)| \geq \gamma\}$; thus $X \cap \X_\gamma$ is the set of points on which our labels will be judged. We further divide this set by label: for $s \in \{-1,+1\}$, let $\X^s_\gamma = \{x \in \X: s \cdot \eta(x) \geq \gamma\}$. 
\begin{enumerate}
\item[(A1)] There is an absolute constant $p_o > 0$ for which the following holds: for any $x \in \X_\gamma$, there exists $B \in \B(x)$ such that $B \subset \X^{s(x)}_\gamma$ and $\mu(B) \geq p_o$.
\end{enumerate}
As we will see, this assumption holds if the decision regions have bounded curvature. It allows us to bound $L_1(x)$ easily.
\begin{lemma}
Under (A1), every $x \in \X_\gamma$ has $L_1(x) \leq \lceil \lg (1/p_o) \rceil$.
\label{lemma:L1-bound}
\end{lemma} 
The second assumption is a variant of the Tsybakov margin condition.
\begin{enumerate}
\item[(A2)] There are absolute constants $C > 0$ and $0 < \sigma < 1$ for which the following holds: for any $p > 0$, we have $\mu(\partial_{p,p}) \leq C p^\sigma$.
\end{enumerate}
If, for instance, $\mu$ were the uniform distribution over $[0,1]^d$, we would expect $\sigma \sim 1/d$.
The third assumption is similar in spirit, but for points of bias $\geq \gamma$. There are two options: (A3) or (A3').
\begin{enumerate}
\item[(A3)] There is an absolute constant $p_1 > 0$ such that $\mu(\partial_p \cap \X_\gamma) = 0$ for any $p < p_1$.
\item[(A3')] There are constants $C' > 0$ and $0 < \xi < 1$ such that $\mu(\partial_{p} \cap \X_\gamma) \leq C' p^\xi$ for any $p > 0$.
\end{enumerate}

Under these assumptions, we can give concrete label complexity bounds.
\begin{thm}
Assume (A1), (A2) and either (A3) or (A3'). There is an absolute constant $C''$ for which the following holds. Pick $0 < \delta < 1$ and take $k = O(((d \log n) + \log (1/\delta))/\gamma^2)$. If the algorithm of Figure~\ref{alg:main} makes $m$ queries, for
$$ \frac{128k}{p_o} \ \leq \ m \ \leq \ \frac{64C \cdot 8^\sigma}{1-\sigma} \, n^{1-\sigma} k^\sigma,$$
then with probability at least $1-\delta$:
\begin{itemize}
\item Under (A3), all of $X \cap \X_\gamma$ get Bayes-optimal labels for $m > (512C/(1-\sigma)) \cdot (1/p_1)^{1-\sigma} \cdot k$.
\item Under (A3'),  Bayes-optimal labels get assigned to all but $ C'' (k/m)^{\xi/(1-\sigma)} + (2/n) \log (4/\delta)$
%$$ C'' \left(\frac{k}{m}\right)^{\xi/(1-\sigma)} + \frac{2 \log (4/\delta)}{n} $$
fraction of $X \cap \X_\gamma$.
\end{itemize}
\label{thm:label-complexity-specific}
\end{thm}

Next, we'll see two settings, here and in the Appendix, in which (A1)-(A3) follow from more commonplace assumptions in nonparametric estimation and computational geometry.

\subsection{Example: Curvature and Massart noise}

Consider the following two conditions. The first says that the distribution over $\X$ is close (within a multiplicative factor) to uniform, while the second says that all of $\X$ has bias $\geq \gamma$ except for a $(d-1)$-dimensional decision surface of bounded curvature.
\begin{enumerate}
\item[(C1)] [Strong density condition] Distribution $\mu$ on $\X \subset \R^d$ admits a density that is bounded below and above: there exist constants $c_o, c_1 > 0$ such that for all balls $B \in \B$,
$$ c_o \vol(B) \leq \mu(B) \leq c_1 \vol(B) ,$$
where $\vol(\cdot)$ is $d$-dimensional volume.
\item[(C2)] [Massart noise and boundary condition] $\X = \X_\gamma^+ \cup \X_\gamma^- \cup \X^0$, where $\X^0$ separates (intersects any line between) $\X_\gamma^+$ and $\X_\gamma^-$, and is a $(d-1)$-dimensional Riemannian submanifold of reach $r_o > 0$.
\end{enumerate}
The \emph{reach} condition says that any point at Euclidean distance $< r_o$ of $\X^0$ has a unique nearest neighbor on the separator. It is a commonly-used notion of curvature in the computational geometry literature~\cite{F59,NSW06} and implies, for instance, that $\X^+$ (resp., $\X^-$) can be covered by open balls of radius $r_o$ that do not touch $\X^- \cup \X^0$ (resp., $\X^+ \cup \X^0$).
\begin{lemma}
Conditions (C1) and (C2) yield assumptions (A1), (A2), and (A3), with $p_o = r_o^d \cdot c_o \cdot v_d$, where $v_d$ is the volume of the unit ball in $\R^d$, and $\sigma = \xi = 1/d$.
\label{lemma:massart-dist}
\end{lemma}

Theorem~\ref{thm:label-complexity-specific} then implies a convergence rate of $1/m^{1/(d-1)}$ after $m$ queries. This is an improvement over the usual $1/m^{1/d}$ rate for random querying and has been observed previously in similar settings~\cite{CN08}.

%\bibliography{./refs}
%\bibliographystyle{plain}

\appendix

\section{Technicalities: discrete setting}

\subsection{Large deviation bounds for the discrete setting}

\begin{lemma}
Fix a confidence parameter $0 < \delta < 1$ and a positive integer $k \geq 6 \ln (4/\delta)$. 

Let $x_1, \ldots, x_m$ be any collection of points. Suppose that the labels $Y_i \in \{-1,1\}$ of these points are independent, with $\E Y_i = \eta(x_i)$. Define
$$ \eta_o = \frac{1}{m} \left( \eta(x_1) + \cdots + \eta(x_m) \right) .$$
Now consider the following estimator $Z$ of this quantity:
\begin{itemize}
\item Each $x_i$ is chosen with probability $q > 0$, independently. Let $N$ be the number of selected points.
\item If $N > 0$, the labels $Y_i$ of the selected points are obtained, and $Z$ is their average.
\end{itemize}
If $qm \geq k + \sqrt{6k \ln (4/\delta)}$, with probability at least $1-\delta$, 
\begin{enumerate}
\item[(a)] $N \geq k$, and 
\item[(b)] $| Z - \eta_o| < \sqrt{(48/k) \ln (4/\delta)}$.
\end{enumerate}
\label{lemma:large-deviation-discrete}
\end{lemma}
\begin{proof}
Let's start with (a). Define $c = \sqrt{3 \ln (4/\delta)}$. We'll take $qm = k + \sqrt{6k \ln (4/\delta)} = k + c \sqrt{2k}$ since this is the worst case. By assumption, $k \geq 2c^2$ and thus $qm \leq 2k$.

Now, $N$ has a binomial($m,q$) distribution. By a multiplicative Chernoff bound, for $0 < \epsilon < 1$, we have
\begin{align*}
\pr(N \geq qm(1+\epsilon)) &\leq e^{-qm\epsilon^2/3} \\
\pr(N \leq qm(1-\epsilon)) &\leq e^{-qm\epsilon^2/2}
\end{align*}
Take $\epsilon = c/\sqrt{qm}$; by the lower bound on $k$, we have $\epsilon \leq 1/2$. Recalling the choice of $c$, we see that with probability at least $1-\delta/2$, we get 
$$(1-\epsilon) qm < N < (1+\epsilon) qm .$$
The lower bound implies $N > qm(1-\epsilon) = qm - c\sqrt{qm} = k + c \sqrt{2k} - c \sqrt{qm} \geq k$.

For (b), define $C_1, \ldots, C_m \in \{0,1\}$ as random variables indicating whether the corresponding points were selected; that is, $C_i = {\bf 1}(\mbox{$x_i$ was selected})$. The sum of the obtained labels is then $S = C_1 Y_1 + \cdots + C_m Y_m$. Notice that these $C_iY_i \in \{-1,0,1\}$ are independent with $\E[C_iY_i] = q \eta(x_i)$ and $\E[(C_iY_i)^2] = \E[C_i] = q$. Thus their sum $S$ has expectation
$$ \E [S] = \sum_{i=1}^m \E[C_i Y_i] = \sum_{i=1}^m q \eta(x_i) = qm \eta_o $$
and variance
$$ \mbox{var}(S) = \sum_{i=1}^m \mbox{var}(C_iY_i) \leq qm .$$
We can bound the concentration of $S$ around its expected value using Bernstein's inequality, by which
$$ \pr(|S - \E[S]| \geq t) \leq 2 \exp \left( - \frac{t^2}{2(\mbox{var}(S) + t/3)} \right) .$$
Using $t = \epsilon qm$, we then have that $|S - qm \eta_o| \leq \epsilon qm$ with probability at least $1-\delta/2$.

Combining with the high-probability bound on $N$ above, we get
$$ \frac{qm \eta_o - \epsilon qm}{qm(1+\epsilon)} < \frac{S}{N} < \frac{qm \eta_o + \epsilon qm}{qm(1-\epsilon)},$$
whereupon (recalling $Z = S/N$)
$$ 
\eta_o \left( \frac{1}{1+\epsilon} -1 \right) - \frac{\epsilon}{1+\epsilon} < Z - \eta_o < \eta_o \left( \frac{1}{1-\epsilon} - 1 \right) + \frac{\epsilon}{1-\epsilon} .$$
Since $|\eta_o| \leq 1$,
$$ |Z - \eta_o | < \max \left( \frac{2\epsilon}{1+\epsilon}, \frac{2\epsilon}{1-\epsilon} \right)
\leq 
4 \epsilon,$$
as claimed.  
\end{proof}

\subsection{Accuracy of bias estimates: Proof of Lemma~\ref{lemma:accurate-bias-estimates}}

We will use Lemma~\ref{lemma:large-deviation-discrete} to obtain a uniform guarantee on the bias estimates for all balls $B \in \B$.

Recall that each point $x \in X$ gets a label $Y \in \{-1,+1\}$ according to the distribution
$$ \eta(x) = \E[y | x].$$ 
For any $B \in \B$ with $X_B \neq \emptyset$, let $\eta_X(B)$ be the average $\eta$-value of the points in $B$, that is,
$$ \eta_X(B) = \frac{1}{|X_B|} \sum_{x \in X_B} \eta(x) .$$

For any $B \in \B_\ell$, define its \emph{query-set} to be $\Gamma(B) = \{z \in X_B: T_z \leq \tau_\ell\}$, where the sampling threshold $\tau_\ell$ for level $\ell$ is defined as:
\begin{equation}
\tau_\ell = \min \left( \frac{2^{\ell+2}k}{n}, \ 1 \right)
\end{equation}
for some $k$. We base our bias-estimate for $B$ on the labels of points in $\Gamma(B)$.

For what follows, define
\begin{equation*}
c = \sqrt{3 \ln \frac{4|\B|}{\delta}} .
\end{equation*}

\begin{lemma}
Suppose $k \geq 2c^2$. With probability at least $1-\delta$, the following is true for all $B \in \B$ with $|X_B| \geq k$:
\begin{enumerate}
\item[(a)] The query set $\Gamma(B)$ has size at least $k$.
\item[(b)] The average label on $\Gamma(B)$, call it $\widehat{\eta}(B)$, satisfies
$$ \left| \widehat{\eta}(B) - \eta_X(B) \right| \leq \frac{4c}{\sqrt{k}}.$$
\end{enumerate}
\label{lemma:large-deviation-bounds}
\end{lemma}
\begin{proof}
Pick $B \in \B$. There are two cases to consider.

Case 1: $|X_B| \geq 2k$. The choice of $\tau_\ell$ then ensures $|X_B| \tau_\ell \geq 2k$, so that $|\widehat{\eta}(B) - \eta_X(B)|$ can be bounded by applying Lemma~\ref{lemma:large-deviation-discrete} to the points $X_B$ with sampling probability $q = \tau_\ell$.

Case 2: $k \leq |X_B| < 2k$. In this case, $B$ lies at a level $\ell$ for which $\tau_\ell = 1$. Lemma~\ref{lemma:large-deviation-discrete} does not apply directly, but its conclusion still holds. In particular, the query-set $\Gamma(B)$ is all of $X_B$, and the same Bernstein bound from the proof of Lemma~\ref{lemma:large-deviation-discrete} can be again be applied.

To complete the proof, we take a union bound over all $B \in \B$.
\end{proof}

The following corollary of Lemma~\ref{lemma:large-deviation-bounds} is immediate.
\begin{cor}
Suppose that $k \geq (8c/\gamma)^2$. For each $B \in \B$, let $\widehat{\eta}(B)$ be the average label on the query-set $\Gamma(B)$, and define the bias-estimate $\yh(B)$ as follows:
$$ \yh(B)
= 
\left\{
\begin{array}{ll}
\sign(\widehat{\eta}(B)) & \mbox{if $|\widehat{\eta}(B)| \geq \gamma/2$} \\
0 & \mbox{otherwise}
\end{array}
\right.
$$
With probability $\geq 1-\delta$, all bias estimates $\yh(B)$, for $B \in \B$, are $\gamma$-accurate.
\label{cor:accurate-bias-estimates}
\end{cor}
\begin{proof}
By the choice of $k$, we have from Lemma~\ref{lemma:large-deviation-bounds} that 
$| \widehat{\eta}(B) - \eta_X(B) | < \gamma/2$
for all $B \in \B$.
\end{proof}

\subsection{Critical levels: Proof of Lemma~\ref{lemma:boundary}}

For part (a), note that some $B_o \in \B_{L_1(x)}(x)$ has significant bias (that is, bias $\geq \gamma$) towards the correct label $s(x)$, as does any ball $B \in \B(x)$ contained within it. For $\ell \geq L_1(x)$, the set $\{B \in \B_{\leq \ell}(x): X_{B} \subseteq X_{B_o}\}$ is nonempty (it contains $B_o$); pick any minimal ball $B$ within it. Then $\yh(B) = s(x)$ by the $\gamma$-accuracy of bias estimates (Lemma~\ref{lemma:accurate-bias-estimates}) and thus $s(x) \in \PL_\ell(x)$.

For (b), take any $\ell \geq L_2(x)$. Consider any $B \in \B_{\leq \ell}(x)$ for which $s(x) \cdot \eta_X(B) < 0$. By definition of $L_2(x)$, this $B$ must lie in $B_{< L_2(x)}(x)$, and moreover there must exist $B' \in \B_{\leq L_2(x)}(x) \subset \B_{\leq \ell}(x)$ with $X_{B'} \subset X_B$ and $s(x) \cdot \eta_X(B') \geq 0$. Thus, any minimal $B \in \B_{\leq \ell}(x)$ has bias $\geq 0$ towards the correct label $s(x)$, whereupon $\yh(B) \in \{0, s(x)\}$ by the $\gamma$-accuracy of bias estimates. Therefore, $-s(x) \not\in \PL_\ell(x)$.

\subsection{The region of focused sampling: Proof of Lemma~\ref{lemma:focused}}

Let $\overline{U}_\ell$ denote the set of all points that are ever (in any round of sampling) in the uncertainty set at level $\ell$. From Lemma~\ref{lemma:boundary}(b), we see that any $x$ with $L_2(x) < \ell$ has $\yh_{\ell-1}(x) \neq \ !$ and thus never makes it into the uncertainty set at level $\ell$. In short,
\begin{equation}
\overline{U}_{\ell} \subset \{x \in X: L_2(x) \geq \ell\}.
\label{eq:uncertainty}
\end{equation}

We see from the {\tt Focused-query} subroutine (Figure~\ref{alg:sampling}) that all focused samples at level $\ell$ lie in
$$
\bigcup_{x \in \overline{U}_\ell} \bigcup_{B \in \B_\ell(x)} \{z \in X_B: T_z \leq \tau_\ell \}
\ 
\subset 
\ 
\bigcup_{x \in X: L_2(x) \geq \ell} \bigcup_{B \in \B_\ell(x)} \{z \in X_B: T_z \leq \tau_\ell \}
\ 
=
\ 
\{z \in \Delta_\ell: T_z \leq \tau_\ell \}.
$$

\subsection{A generic label complexity bound}

We start by showing that various subsets of interest contain roughly the expected number of points at each level.
\begin{lemma}
With probability at least $1-2(\lg (n/k))e^{-k/3}$, the following hold for all levels $0 \leq \ell \leq \lg (n/2k)$.
\begin{enumerate}
\item[(a)] $|\{x \in X: T_x \leq \tau_\ell\}| < 2 n \tau_\ell$.
\item[(b)] If $\Delta_\ell \neq \emptyset$ then $|\{x \in \Delta_{\ell}: T_x \leq \tau_\ell\}| < 2 |\Delta_\ell| \tau_\ell$.
\end{enumerate}
\label{lemma:level-distribution}
\end{lemma}
\begin{proof} 
Pick any subset $S \subset X$ and let $m = |S|$. Then $|\{x \in S: T_x \leq \tau_\ell\}|$ has a $\mbox{binomial}(m, \tau_\ell)$ distribution with expectation $m \tau_\ell$. The probability that it is greater than or equal to twice its expected value is, by a multiplicative Chernoff bound, at most $e^{-m \tau_{\ell}/3}$, which is $\leq e^{-k/3}$ as long as $m \tau_\ell \geq k$.

Both parts follow from this principle; and we take a union bound over all $\lg (n/k)$ levels. For (b), we need to check that $|\Delta_\ell| \tau_\ell \geq k$. To see this, observe from the definition (\ref{eq:sampling-region}) of $\Delta_\ell$ that if it is non-empty, then it contains $X_B$ for at least one ball $B \in \B_\ell$, and every such ball has at least $n/2^{\ell+1}$ points. Combining this with the definition $\tau_\ell = \min(2^{\ell+2}k/n, 1)$ yields $|\Delta_\ell| \tau_\ell \geq k$.
\end{proof}

Theorem~\ref{thm:label-complexity} is a restated version of the following.
\begin{thm}
Suppose the active learning algorithm makes $0 < m \leq n$ queries. Then all points $x \in X$ with $L_1(x) \leq \ell_1$ and $L_2(x) \leq \ell_2$ will get Bayes-optimal labels $\yh(x) = g^*(x)$, where
$$ \ell_1 = \left\lfloor \lg \frac{m}{32k} \right\rfloor $$
and $\ell_2$ is the largest integer $\leq \lg (n/2k)$ such that
$$ \sum_{\ell = \ell_1 + 1}^{\ell_2} |\Delta_{\ell}| \, \tau_{\ell} \ < \ \frac{m}{8} .$$
\label{thm:label-complexity-0}
\end{thm}

\begin{proof}
Denote the first $m/2$ queries by {\it phase one} and the second $m/2$ by {\it phase two}. We will analyze the effect of background sampling in phase one and focused sampling in phase two. We start with the former.

Of the $m/2$ queries in phase one, at least $m/4$ will be background samples. Therefore the $m/4$ points with lowest $T_x$ values are guaranteed to be queried.

Now, for $\ell_1$ as defined, we have that $\tau_{\ell_1} \leq m/8n$ and thus from Lemma~\ref{lemma:level-distribution}(a) that at most $m/4$ points in $X$ satisfy $T_x \leq \tau_{\ell_1}$. Therefore all such points are queried in phase one, and all label-estimates $\{\yh_{\ell_1}(x): x \in X\}$ are set.

It then follows from Lemma~\ref{lemma:boundary}(a) that the following hold for any $x \in X$ with $L_1(x) \leq \ell_1$:
\begin{enumerate}
\item[(a)] By the end of phase one, $\yh_{\ell_1}(x) \in \{g^*(x), !\}$.
\item[(b)] For any $\ell > \ell_1$, when $\yh_\ell(x)$ becomes available, it lies in $\{g^*(x), !\}$.
\item[(c)] If $x$ ever leaves the combined uncertainty region $U = \cup_\ell U_\ell$ during phase two, then its final label as defined in (\ref{eq:final-label}) is henceforth always $\yh(x) = g^*(x)$.
\end{enumerate}

Now let's move on to phase two. Let $A = \{x \in X: L_1(x) \leq \ell_1, L_2(x) \leq \ell_2\}$. We will show that every point in $A$ must leave the uncertainty region $U$ at some time during phase two. From (c), we can conclude that all these points have their final labels set correctly, once and for all.

We will break the argument into two cases. 

{\it Case 1:} Fewer than $m/4$ focused queries are made in phase two. This can only happen if some round of sampling has an empty uncertainty set, meaning that all of $A$ has left $U$ at that point.

{\it Case 2:} A full $m/4$ focused queries are made in phase two. By the analysis of phase one, none of these queries can be at level $\leq \ell_1$ and by Lemma~\ref{lemma:focused}, the total number of possible focused queries at levels $\ell_1+1$ through $\ell_2$ inclusive is at most
$$\sum_{\ell=\ell_1+1}^{\ell_2} |\{z \in \Delta_{\ell}: T_z \leq \tau_\ell \}| 
\ \leq \ 
\sum_{\ell=\ell_1+1}^{\ell_2} 2 |\Delta_{\ell}| \tau_\ell
\ < \ 
\frac{m}{4} ,
$$
where the first inequality is from Lemma~\ref{lemma:level-distribution}(b). Thus at least one query in phase two must be at level $\ell_2 +1$. When this query is made, every $U_\ell$ with $\ell \leq \ell_2$ must be empty and thus all of $A$ must have left the uncertainty region; recall from (\ref{eq:uncertainty}) that no $x \in A$ can be part of $U_\ell$ for $\ell > \ell_2 \geq L_2(x)$. 

Thus every $x \in A$ must leave the uncertainty region at some point in phase two, and their final labels are subsequently set correctly.
\end{proof}

\subsection{One-dimensional monotonic $\eta$: Proof of Theorem~\ref{thm:oned-monotonic}}

We begin by bounding the critical levels $L_1$ and $L_2$ for points in $X$.
\begin{lemma}
Pick any $x \in [0, 1]$.
\begin{enumerate}
\item[(a)] Define $n^- = |[0,\lambda_L] \cap X|$ and $n^+ = |[\lambda_R,1] \cap X|$. Then
$$
L_1(x)
\leq
\left\{
\begin{array}{ll}
\lg (n/n^+) & \mbox{if $x \geq \lambda_R$} \\
\lg (n/n^-) & \mbox{if $x \leq \lambda_L$}
\end{array}
\right.
$$
\item[(b)] Let $r(x)$ be the number of points between $x$ and the boundary point $\lambda$, counting $x$ itself. That is, $r(x) = |[x,\lambda) \cap X|$ if $x < \lambda$, or $|(\lambda, x] \cap X|$ if $x > \lambda$. Then $L_2(x) \leq \lg (n/r(x))$.
\end{enumerate}
\label{lemma:oned-monotonic-L}
\end{lemma}
\begin{proof}
For (a), take any $x \geq \lambda_R$ (the other case is similar). The interval $B = [\lambda,1]$ lies in $\B_\ell(x)$ for $\ell = \lceil \lg (n/n^+) - 1 \rceil$ and has $\eta_X(B) \geq \gamma$. Furthermore, any $B' \subset B$ also has $\eta_X(B') \geq \gamma$. 

For (b), take $x \geq \lambda_R$ and $\ell \geq \lg (n/r(x))$. Any $B \in \B_\ell(x)$ contains $< n/2^\ell \leq r(x)$ points and thus cannot possibly extend to the other side of the boundary. It follows that every $B \in \B_{\geq \ell}(x)$ has $\eta_X(B) \geq 0$. Moreover, any interval $B'$ that does extend to the other side of the boundary contains $B' \in \B_{\ell}(x)$ that is entirely on the same side as $x$.
\end{proof}

We can now bound the size of the query region at each level and find that it shrinks exponentially with $\ell$.
\begin{lemma}
For any $\ell \geq 0$, let $\Delta_\ell$ denote the focused querying region at level $\ell$, as defined in (\ref{eq:sampling-region}). Then $|\Delta_\ell| \leq 4n/2^\ell$.
\label{lemma:oned-monotonic-query-region}
\end{lemma}

\begin{proof}
We have
\begin{align*}
\Delta_\ell 
&= \bigcup_{x \in X: L_2(x) \geq \ell} \bigcup_{B \in \B_{\ell}(x)} (B \cap X) \\
&\subset \bigcup_{x \in X: r(x) \leq n/2^\ell} \bigcup_{B \ni x: |B \cap X| < n/2^\ell} (B \cap X) 
.
\end{align*}
This includes at most $n/2^{\ell-1}$ points from $X$ on either side of $\lambda$. \end{proof}

We are now ready for the proof of Theorem~\ref{thm:oned-monotonic}.

Setting $k$ to $O((1/\gamma^2) \ln (n/\delta))$ satisfies the requirements of Theorem~\ref{thm:label-complexity}. Here we are using the fact that although $\B$ is infinite, we need only consider $O(n^2)$ distinct intervals since $|X| = n$.

First observe that for any $\ell_1 \leq \ell_2$, we have from Lemma~\ref{lemma:oned-monotonic-query-region} that
$$ \sum_{\ell=\ell_1+1}^{\ell_2} |\Delta_\ell| \tau_\ell 
\ \leq \ \sum_{\ell=\ell_1+1}^{\ell_2} \frac{4n}{2^\ell} \cdot \frac{2^{\ell+2} k}{n} 
\ = \ 16k(\ell_2-\ell_1).$$
Now, let's define
$$ \ell_1 = \left\lfloor \lg \frac{m}{32k} \right\rfloor, 
\ \ \ \ell_2 = \min \left( \ell_1 + \frac{m}{128k}, \ \ \lg \frac{n}{2k} \right).$$
Then for any $x \in [0, \lambda_L] \cup [\lambda_R,1]$, we have
$$ \ell_1 = \left\lfloor \lg \frac{m}{32k} \right\rfloor 
\ \geq \ 
\lg \frac{m}{64k}
\ \geq \ 
\lg \frac{n}{\min(n^+,n^-)}
\ \geq \ 
L_1(x)$$
and, if $\min(r^+, r^-) \geq 2k$,
\begin{align*}
\ell_2 
\ = \ \ell_1 + \frac{m}{128k} 
&\geq \lg \frac{m}{64k} + \lg \frac{\min(n^+,n^-)}{\min(r^+,r^-)} \\
&\geq \ \lg \frac{m}{64k} + \lg \frac{64kn/m}{\min(r^+,r^-)} 
\ = \ \lg \frac{n}{\min(r^+,r^-)}
\ \geq \  
L_2(x).
\end{align*}
We get the algorithmic guarantee by applying Theorem~\ref{thm:label-complexity}. 

\subsection{One-dimensional data with Massart noise}
\label{sec:oned-massart}

We now turn to another one-dimensional setting. Once again, $X$ consists of $n$ arbitrarily-placed points in $\X = [0,1]$. This time, however, they are labeled according to a conditional probability function $\eta: \X \to [-1,1]$ that satisfies the Massart noise condition:
\begin{itemize}
\item There are $p$ disjoint open intervals $I_1, \ldots, I_p$, such that $\X$ is (the closure of) their union, and
\item for each $j$, either $\eta(x) > \gamma$ for all $x \in I_j$ or $\eta(x) < -\gamma$ for all $x \in I_j$. In the first case, we write $s(I_j) = +1$ and in the second case, $s(I_j) = -1$.
\end{itemize}
Here $0 < \gamma < 1$ is some constant. See Figure~\ref{fig:oned}(b) for an illustrative example. For concreteness, the intervals $I_j$ can be written in the form $(\lambda_{j-1}, \lambda_j)$, where
$0 = \lambda_0 < \lambda_1 < \cdots < \lambda_{p-1} < \lambda_p = 1 .$
Here $\lambda_1, \ldots, \lambda_{p-1}$ are the \emph{boundary points} between intervals.

We will take $\B$ to consist of all open intervals of $[0,1]$, with $\B(x)$ denoting intervals that contain point $x$. 

We begin with bounds on the $L_1$ and $L_2$ levels for each point.
\begin{lemma}
Pick any $x \in \X$; suppose $x \in I_j$.
\begin{enumerate}
\item[(a)] Let $n_j = |X \cap I_j|$. Then $L_1(x) \leq \lg (n/n_j)$.
\item[(b)] Let $r(x)$ be the minimum number of data points that lie between $x$ and a boundary point, counting $x$ as well; this is either the number of points in the left-interval $(\lambda_{j-1},x]$ (if $j > 1$) or the right-interval $[x, \lambda_j)$ (if $j < p$), whichever is smaller. Then $L_2(x) \leq \lg (n/r(x))$.
\end{enumerate}
\label{lemma:oned-massart-L}
\end{lemma}
\begin{proof}
For (a), notice first that $I_j \in \B_{\ell}(x)$ for $\ell = \lceil (\lg n/n_j) - 1 \rceil$. Moreover, $s(I_j) \cdot \eta_X(I_j) > \gamma$. Thus $I_j$ belongs to $\B_\ell(x)$ and is strongly biased towards the correct label. This strong bias also holds for any subset of $I_j$. 

For (b), consider any $\ell \geq \lg (n/r(x))$. Any $B \in \B(x)$ with $\mbox{sign}(\eta_X(B)) \neq s(I_j)$ must contain either the entire left-interval $(\lambda_{j-1},x]$ or the entire right-interval $[x, \lambda_j)$, and thus has at least $r(x)$ points, which means that it is too large to be in $\B_\ell$. Thus all intervals $B \in \B_{\geq \ell}(x)$ have $s(I_j) \cdot \eta_X(B) > 0$. Also, for any interval $B \in \B_{<\ell}(x)$ there is some $B' \in \B_{\ell}(x)$ that is strictly contained within it.
\end{proof}

With $L_1(x)$ and $L_2(x)$ under control, it is easy to bound the size of the focused query region $\Delta_\ell$ at each level.
\begin{lemma}
For any $\ell \geq 0$, let $\Delta_\ell$ denote the focused querying region at level $\ell$, as defined in (\ref{eq:sampling-region}). Then $|\Delta_\ell| \leq (p-1)n/2^{\ell-2}$.
\label{lemma:oned-massart-query-region}
\end{lemma}

\begin{proof}
We have
\begin{align*}
\Delta_\ell 
&= \bigcup_{x \in X: L_2(x) \geq \ell} \bigcup_{B \in \B_{\ell}(x)} (B \cap X) \\
&\subset \bigcup_{x \in X: r(x) \leq n/2^\ell} \bigcup_{B \ni x: |B \cap X| < n/2^\ell} (B \cap X) 
.
\end{align*}
This includes at most $n/2^{\ell-1}$ points from $X$ on either side of each boundary point $\lambda_j$.
\end{proof}

Notice that $|\Delta_\ell|$ shrinks exponentially with $\ell$. With $L_1$, $L_2$, and $|\Delta_\ell|$ values in place, Theorem~\ref{thm:label-complexity} can be applied directly to give the following label complexity bound.

\begin{thm}
Pick any $0 < \epsilon, \delta < 1$. Suppose we run the algorithm of Figure~\ref{alg:main} with $k = O((1/\gamma^2) \ln (n/\delta))$. With probability at least $1-\delta$, after making
$$ O \left( \frac{p-1}{\gamma^2} \ln \frac{p-1}{\epsilon} \ln \frac{n}{\delta} \right) $$
queries, the algorithm will assign the correct label $g^*(x)$ to at least $1-\epsilon$ fraction of $X$, except possibly the $2k$ points of either side of each boundary point.
\label{thm:oned-massart}
\end{thm}

\begin{proof}
Setting $k$ to $O((1/\gamma^2) \ln (n/\delta))$ satisfies the requirements of Theorem~\ref{thm:label-complexity}. Here we are using the fact that although $\B$ is infinite, we need only consider $O(n^2)$ distinct intervals since $|X| = n$.

Next, using Lemma~\ref{lemma:oned-massart-query-region}, we have that for any integers $0 \leq \ell_1 < \ell_2$,
$$ \sum_{\ell = \ell_1+1}^{\ell_2} |\Delta_\ell| \tau_\ell
\ \leq \ 
\sum_{\ell = \ell_1+1}^{\ell_2} \frac{(p-1) n}{2^{\ell-2}} \cdot \frac{2^{\ell+2}k}{n}
\ = \ 
16(p-1)k (\ell_2 - \ell_1).
$$
We can then apply Theorem~\ref{thm:label-complexity} to conclude that $m$ query points are enough to correctly classify all $x \in X$ with $L_1(x) \leq \ell_1$ and $L_2(x) \leq \ell_2$, for
\begin{align*}
\ell_1
&= 
\left\lfloor \lg \frac{m}{32k} \right\rfloor \\
\ell_2
&=
\min \left( \ell_1 + \frac{m}{128 k(p-1)}, \ \lg \frac{n}{2k} \right)
\end{align*}
Using Lemma~\ref{lemma:oned-massart-L}, we have that in every target interval $I_j$ with $n_j/n = \Omega(k/m)$, all but $n \cdot 2^{-\ell_2-1}$ points will be correctly classified. For large enough $m$, this means that the fraction of misclassified or unclassified points in $X$ will be at most $\epsilon$ after $O(k(p-1) \log ((p-1)/\epsilon))$ queries, apart from the $2k$ points nearest the boundaries, which will remain unclassified.
\end{proof}

The $2k$ points nearest each boundary cannot be labeled by our algorithm with any certainty because they lie in intervals with a strongly positive bias as well as in intervals with a strongly negative bias. This qualification would be removed if were allowed to make multiple queries to each point, because in that case we would include $O(k)$ copies of each point, as explained earlier.

\section{Technicalities: distributional setting}

In the distributional setting, $X$ is drawn i.i.d.\ from a distribution $\mu$ on $\X \subset \R^d$.  We assume $\mu$ is absolutely continuous with respect to the Lebesgue measure on $\R^d$ and thus admits a density. Take $\B$ to consist of all open balls $B(x,r) = \{z: \|z-x\| < r\}$, with $\B(x)$ being balls that contain $x$.

\subsection{Sampling level and probability mass}

To begin with, we relate the number of points in a ball $B \in \B$ (and thus the level of the ball) to its probability mass under the marginal distribution.
\begin{lemma}
With probability at least $1-\delta$, for all $B \in \B$ with $n \mu(B) \geq 12 \ln (2|\B|/\delta)$, we have
$$ \frac{\mu(B)}{2} \leq \frac{|X_B|}{n} \leq 2 \mu(B) .$$
\label{lemma:ball-size-bounds}
\end{lemma}
\begin{proof}
It is an immediate consequence of the multiplicative Chernoff bound that with probability at least $1-\delta$, for all $B \in \B$,
$$ \frac{|X_B|}{n} = \mu(B) \left(1 \pm \sqrt{\frac{3}{n \mu(B)} \ln \frac{2|\B|}{\delta}} \right) .$$
\end{proof}

Henceforth assume that this high-probability event holds. Next, we will see that balls of probability mass $p$ belong to level $\ell \approx \lg (1/p)$.
\begin{lemma}
For any level $\ell \geq 0$ and any $B \in \B_{\ell}$,
$$ \left\lceil \lg \frac{1}{\mu(B)} \right\rceil -2 \leq \ell \leq  \left\lceil \lg \frac{1}{\mu(B)} \right\rceil.$$
\label{lemma:mass-vs-level}
\end{lemma}
\begin{proof}
Recall the definition of level $\ell$:
$$ B \in \B_\ell 
\ \Longleftrightarrow \ \frac{n}{2^{\ell+1}} \leq |X_B| < \frac{n}{2^\ell} 
\ \Longleftrightarrow \ 2^\ell < \frac{n}{|X_B|} \leq 2^{\ell + 1}.$$
By Lemma~\ref{lemma:ball-size-bounds}, 
$$ \frac{1}{2 \mu(B)} \leq \frac{n}{|X_B|} \leq \frac{2}{\mu(B)}.$$
Thus we must have $2^{\ell} < 2/\mu(B)$ and $1/(2 \mu(B)) \leq 2^{\ell+1}$. These translate into the stated bounds on $\ell$.
\end{proof}

\subsection{Bounding $L_2$ using probabilistic distance: Proof of Lemma~\ref{lemma:L2-bound}}

Let $p = \dist(x, \X^{-s(x)})$. If $p = 0$, the statement is vacuous, so assume $p > 0$.

Consider $\ell = \lceil \lg (1/p) \rceil + 1$. For any $B \in \B_{\geq \ell}(x)$, we have $\mu(B) \leq 2|X_B|/n < 2/2^\ell \leq p$, using Lemma~\ref{lemma:ball-size-bounds}, the definition of sampling levels, and the definition of $\ell$, in that order. It follows that $B$ does not intersect $\X^{-s(x)}$, whereupon $s(x) \cdot \eta_X(B) \geq 0$.

Next, pick any $B \in \B_{\leq \ell}(x)$ with $s(x) \cdot \eta_X(B) < 0$. Thus $B$ must intersect $\X^{-s(x)}$ and has probability mass $\geq p$. We will show that there exists $B' \in \B_{\leq \ell}(x)$ such that $B' \subset B$ and $B'$ does not intersect $\X^{-s(x)}$; whereupon $s(x) \cdot \eta_X(B') \geq 0$. Indeed, take $B' \in \B(x)$ to be a subset of $B$ of $\mu$-mass $p-\epsilon$ for some very small $\epsilon$; we can do this because of the absolute continuity of $\mu$. Then $B'$ does not touch $\X^{-s(x)}$ and by Lemma~\ref{lemma:mass-vs-level}, for small enough $\epsilon$, it lies at level $\leq \ell$.

\subsection{Uncertainty region in the distributional setting: Proof of Lemma~\ref{lemma:delta-continuous}}

Recall from (\ref{eq:sampling-region}) that
$$ \Delta_\ell = \bigcup_{x \in X: L_2(x) \geq \ell} \bigcup_{B \in \B_{\ell}(x)} X_B .$$
Consider any $z \in \Delta_\ell$. Then there exists $x \in \X$ with $L_2(x) \geq \ell$ and $B \in \B_\ell(x)$ such that $z \in B$. Now, $B \in \B_\ell(x)$ implies $|X_B|/n < 1/2^\ell$ and so (by Lemma~\ref{lemma:ball-size-bounds}) $\mu(B) < 2/2^\ell$. We will show that $x \in \partial_{4/2^\ell}$ and thus $z \in \partial_{4/2^\ell, 2/2^{\ell}}$.

There are two cases for $x$. If $\eta(x) = 0$ then we immediately have $x \in \partial_{4/2^\ell}$. Otherwise, $\eta(x) \neq 0$. In this case, since $L_2(x) \geq \ell$, we can apply Lemma~\ref{lemma:L2-bound} to get 
$$ \ell 
\ \leq \ \left\lceil \lg \frac{1}{\dist(x,\X^{-s(x)})} \right\rceil + 1 
\ \leq \ 
\lg \frac{1}{\dist(x,\X^{-s(x)})} + 2 .
$$
Thus $\dist(x, \X^{-s(x)}) \leq 1/2^{\ell-2}$ and $x \in \partial_{4/2^\ell}$.

\subsection{A generic label complexity bound in the distributional setting}

We will need to relate the size of the query region $\Delta_\ell$ to the probability mass of the corresponding second-order boundary. For this, we provide an analog of Lemma~\ref{lemma:level-distribution} for the distributional setting.
\begin{lemma}
With probability at least $1-2(\lg (n/k))e^{-k/3}$, the following hold for all levels $0 \leq \ell \leq \lg (n/2k)$.
\begin{enumerate}
\item[(a)] $|\{x \in X: T_x \leq \tau_\ell\}| \leq 2 n \tau_\ell$.
\item[(b)] $|\{x \in \Delta_{\ell}: T_x \leq \tau_\ell\}| \leq 2 n \mu(\partial_{4/2^\ell, 2/2^\ell}) \tau_\ell$.
\end{enumerate}
\label{lemma:level-distribution-dist}
\end{lemma}
\begin{proof} 
Part (a) is as in Lemma~\ref{lemma:level-distribution}. 

For (b), pick any subset $S \subset \X$. If $X$ consists of $n$ independent draws from $\mu$, then $|\{x \in S: T_x \leq \tau_\ell\}|$ has a $\mbox{binomial}(n, \mu(S) \tau_\ell)$ distribution with expectation $n \mu(S) \tau_\ell$. The probability that it is more than twice its expected value is, by a multiplicative Chernoff bound, at most $e^{-n \mu(S) \tau_\ell/3}$. We will apply this to the various sets $S = \partial_{4/2^\ell, 2/2^\ell}$ and take a union bound over them. In each application, we will also see that $n \mu(S) \tau_\ell \geq k$.

Pick any $\ell \leq \lg (n/2k)$. If $\partial_{4/2^\ell, 2/2^\ell} = \emptyset$, then the statement in (b) is trivially true given Lemma~\ref{lemma:delta-continuous}. So assume this is not the case. Writing $p = 4/2^\ell$, we need to check that $\mu(\partial_{p, p/2}) \tau_\ell \geq k/n$, or equivalently, $\mu(\partial_{p, p/2}) \geq \max(1/2^{\ell+2}, k/n)$. Now, $\partial_{p,p/2} \neq \emptyset \Longrightarrow \partial_p \neq \emptyset$. Pick any $x \in \partial_p$. By absolute continuity of $\mu$, we can grow a ball around $x$ of probability mass arbitrarily close to $p/2$, so that this ball is contained within $\partial_{p,p/2}$. Thus $\mu(\partial_{p,p/2}) \geq p/2-\epsilon$ for any $\epsilon > 0$. The required conditions then follow from the value of $p$.
\end{proof}

Theorem~\ref{thm:label-complexity} now takes on the following form. Since the set of balls in $\R^d$ has VC dimension $d+1$, we can take $|\B|$ to be $O(n^{d+1})$.
\begin{thm}
Suppose that $k \geq (192/\gamma^2) \ln (4 |\B|/\delta)$ and that the active learning algorithm makes $0 < m \leq n$ queries. Then with probability at least $1-3\delta$, all points $x \in X$ with $L_1(x) \leq \ell_1$ and $L_2(x) \leq \ell_2$ will get Bayes-optimal labels $\yh(x) = g^*(x)$, where 
$$ \ell_1 = \left\lfloor \lg \frac{m}{32k} \right\rfloor $$
and $\ell_2$ is the largest integer $\leq \lg (n/2k)$ such that
$$ \sum_{\ell = \ell_1 + 1}^{\ell_2} 2^\ell \mu(\partial_{4/2^\ell, 2/2^\ell}) \ < \ \frac{m}{32k} .$$
\label{thm:label-complexity-dist}
\end{thm}

\begin{proof}
The proof is identical to that of Theorem~\ref{thm:label-complexity-0}; the only change is to use Lemma~\ref{lemma:level-distribution-dist}(b) in place of Lemma~\ref{lemma:level-distribution}(b).
\end{proof}

\subsection{Bounding $L_1$ by curvature: Proof of Lemma~\ref{lemma:L1-bound}}

Pick any $x \in X_\gamma$; apply (A1) to get $B \in \B(x)$ for which $\mu(B) \geq p_o$ and $B \subset \X^{s(x)}_\gamma$. By Lemma~\ref{lemma:mass-vs-level}, $B \in \B_\ell$ for $\ell \leq \lceil \lg (1/p_o) \rceil$. Now, $s(x) \cdot \eta_X(B) \geq \gamma$; moreover, for any $B' \in \B(x)$ with $X_{B'} \subset X_B$ we have $X_{B'} \subset \X^{s(x)}_\gamma$ and thus $s(x) \cdot \eta_X(B') \geq \gamma$ as well.

\subsection{Label complexity under three assumptions: Proof of Theorem~\ref{thm:label-complexity-specific}}

In this case, $\B$ is infinite, but by standard VC-dimension arguments there are only $O(n^{d+1})$ balls with distinct sets $X_B$. This governs the setting of $k$.

First, define $\ell_1 = \lfloor \lg (m/32k) \rfloor$ and observe that by Lemma~\ref{lemma:L1-bound}, all points in $\X_\gamma$ have 
$$ L_1(x) 
\ \leq \ 
\left\lceil \lg \frac{1}{p_o} \right\rceil
\ \leq \ 
\lg \frac{2}{p_o} 
\ \leq \ 
\lg \frac{m}{64k}
\ \leq \ 
\ell_1 .
$$
Next, pick  
$$ \ell_2 = \left\lfloor \frac{1}{1-\sigma} \left( \lg \frac{m}{32k} + \lg \frac{1-\sigma}{C \cdot 4^{1+\sigma}} \right) \right\rfloor.$$
The upper bound on $m$ ensures that this is at most $\lg (n/2k)$. Then, using (A2),
$$
\sum_{\ell = \ell_1+1}^{\ell_2} 2^\ell \mu(\partial_{4/2^\ell, 2/2^\ell})
\ \leq \ 
C \sum_{\ell = \ell_1+1}^{\ell_2} 2^\ell \left( \frac{4}{2^\ell} \right)^\sigma 
\ \leq \ 
C \cdot 4^\sigma \cdot \frac{2}{2^{1-\sigma}-1} \cdot 2^{(1-\sigma)\ell_2}
\ \leq \ 
\frac{m}{32k} .
$$
Applying Theorem~\ref{thm:label-complexity-dist}, we then see that all points in $\X_\gamma$ with $L_2(x) \leq \ell_2$ will be correctly classified. Any remaining point $x \in \X_\gamma$ has $L_2(x) > \ell_2$ and thus (by Lemma~\ref{lemma:L2-bound})
$$ \left\lceil \lg \frac{1}{\dist(x, \X^{-s(x)})} \right\rceil + 1 > \ell_2 
\ \Longrightarrow \ 
\dist(x, \X^{-s(x)}) < \frac{4}{2^{\ell_2}} \leq  \left( \frac{512C}{1-\sigma} \right)^{1/(1-\sigma)}  \cdot \left( \frac{k}{m} \right)^{1/(1-\sigma)} . 
$$
Call this quantity $p$; thus any such $x$ lies in $\partial_p$.

Under (A3), $\partial_p \cap \X_\gamma$ has zero probability mass for $p < p_1$, that is, if 
$$ \left( \frac{512C}{1-\sigma} \right)^{1/(1-\sigma)}  \cdot \left( \frac{k}{m} \right)^{1/(1-\sigma)} < p_1
\ \Longleftrightarrow \ 
m > \frac{512C}{1-\sigma} \cdot \frac{1}{p_1^{1-\sigma}} \cdot k .
$$

Under (A3'), $\mu(\partial_p \cap \X_\gamma) \leq C' p^{\xi}$; we can then apply a Bernstein bound to assert that with probability at least $1-\delta$, 
$$ |X \cap (\partial_p \cap \X_\gamma)| \ \leq \ \frac{3}{2} C' n p^{\xi} + 2 \log \frac{1}{\delta},$$
from which the bound in the theorem follows by defining $C''$ appropriately.

\subsection{Label complexity under curvature and Massart noise: Proof of Lemma~\ref{lemma:massart-dist}}

\begin{lemma}
Conditions (C1) and (C2) yield assumption (A1), with $p_o = r_o^d \cdot c_o \cdot v_d$, where $v_d$ is the volume of the unit ball in $\R^d$.
\label{lemma:A1}
\end{lemma}

\begin{proof}
Pick any $x \in \X_\gamma^+$ (the negative case is similar), and let $r = \inf_{z \in \X^0} \|x - z\|$. If $r \geq r_o$, then $B = B(x,r_o)$ is entirely in $\X_\gamma^+$. Otherwise, the reach condition (C2) implies the existence of a ball $B \subset \X_\gamma^+$ that contains $x$ and has radius $r_o$. Either way, $\mu(B) \geq c_o \vol(B) = c_o v_d r_o^d$ by (C1). 
\end{proof}

\begin{lemma}
Conditions (C1) and (C2) yield assumptions (A2) and (A3') with $\sigma = \xi = 1/d$. 
\end{lemma}

\begin{proof}
Under (C1), any ball of probability mass $\leq p$ has volume $\leq p/c_o$ and radius $\leq (p/(c_o v_d))^{1/d}$. Let's call this latter quantity $r$. Thus, any point in $\partial_p$ lies within distance $2r$ of the boundary, while a point in $\partial_{p,p}$ lies within distance $4r$. 

To bound the volume of $\partial_{p,p}$, we can associate each point in this region with its nearest neighbor in $\X^0$; by condition (C1), this projection map is uniquely defined for $r < r_o/4$. The volume of the region is thus $O(r)$ and under (C1), has probability mass $O(r) = O(p^{1/d})$.
\end{proof}

\subsection{Label complexity under smoothness and Tsybakov noise}

Continuing from the Massart setting, we maintain (C1), but now replace (C2) by a trio of smoothness, margin, and curvature conditions. First, recall that $\dist(x,S)$ denotes the probability-distance from $x$ and set $S$. We will overload notation so that for $x,x' \in \X$,
$$ \dist(x,x') = \dist(x, \{x'\}) = \inf\{\mu(B): B \in \B(x) \cap \B(x')\},$$
that is, the probability mass of the smallest ball containing both $x$ and $x'$.
\begin{enumerate}
\item[(C2')] [Holder-smoothness of conditional probability function] There exist constants $L, \alpha$ such that
$$ |\eta(x) - \eta(x')| \ \leq \ L \cdot \dist(x,x')^\alpha$$
for all $x,x' \in \X$.
\item[(C3')] [Tsybakov margin condition] There exists constants $M, \beta$ such that
$$ \mu(\{x \in \X: |\eta(x)| \leq \tau\}) \leq M \tau^\beta$$
for all $\tau \in (0,1)$.
\item[(C4')] [Bounded curvature] The boundaries $\{x \in \X: \eta(x) = \gamma\}$ and $\{x \in \X: \eta(x) = -\gamma\}$ are $(d-1)$-dimensional Riemannian manifolds of reach $r_o > 0$.
\end{enumerate}

Lemma~\ref{lemma:A1} continues to hold, with condition (C4') doing the job of (C2). This yields assumption (A1). For the remaining assumptions, we first obtain a consequence of the Holder condition.

\begin{lemma}
Under (C2'), for any $p, q > 0$,
\begin{enumerate}
\item[(a)] $x \in \partial_p \implies |\eta(x)| \leq L p^\alpha$.
\item[(b)] $x \in \partial_{p,q} \implies |\eta(x)| \leq L(p^\alpha + q^\alpha)$.
\end{enumerate}
\label{lemma:holder-consequence}
\end{lemma}
\begin{proof}
Pick any $x \in \partial_p$. Let $s(x) = \mbox{sign}(\eta(x))$. By definition of the $p$-boundary set, for any $\epsilon > 0$, there exists $x' \in \X^{-s(x)}$ such that $\dist(x,x') < p+\epsilon$. By the Holder condition, $|\eta(x) - \eta(x')| < L(p+\epsilon)^\alpha$ and thus $|\eta(x)| < L(p+\epsilon)^\alpha$. Since this holds for any $\epsilon > 0$, we get part (a).

For (b), pick $x \in \partial_{p,q}$ and $\epsilon > 0$. Then there exists $x' \in \partial_p$ with $\dist(x,x') < q+\epsilon$. As before, we use the Holder condition to conclude that $|\eta(x)| < |\eta(x')| + L(q+\epsilon)^\alpha$ and then invoke (a).
\end{proof}

\begin{lemma}
Conditions (C2') and (C3') yield assumptions (A2) and (A3) with $C = (2L)^\beta M$, $\sigma = \alpha \beta$, and $p_1 = (\gamma/L)^{1/\alpha}$. 
\end{lemma}

\begin{proof}
By Lemma~\ref{lemma:holder-consequence}, $\partial_{p,p} \subset \{x \in \X: |\eta(x)| \leq 2Lp^\alpha\}$; the probability mass of this set can be bounded by (C3').

Also by Lemma~\ref{lemma:holder-consequence}, $\partial_p$ is entirely contained in $\{x \in \X: |\eta(x)| \leq L p^\alpha\}$. For $p < p_1$, this does not intersect $\X_\gamma$.
\end{proof}

\end{document}